\newcommand {\rf} {\mathit{rank}}

\newcommand {\lingconc} {\mathcal{S}}

\newcommand {\nott} {\lnot}

\newcommand {\sx} {\langle}
\newcommand {\dx} {\rangle}

\newcommand {\emme} {\mathcal{M}}

\newcommand {\elle} {\mathcal{L}}

\newcommand {\unione} {\cup}

\newcommand {\tc} {\mid}

\newcommand{\tip}{{\bf T}}

\newcommand{\alc}{\mathcal{ALC}}

\newcommand{\alct}{\mathcal{ALC}+\tip}
\newcommand{\alctmin}{\mathcal{ALC}+\tip_{min}}

\newcommand{\alctr}{\mathcal{ALC}+\tip_{\textsf{\tiny R}}}
\newcommand{\alctre}{{\mathcal{ALC}^{\Ra}\tip}_{E} }

\newcommand{\be}{\begin{enumerate}}
\newcommand{\ee}{\end{enumerate}}

\newcommand{\hide}[1]{}

\def \cases{\left \{\begin{array}{l}}
\def \endcases{\end{array}\right .}

\newcommand {\Ra} {{\bf R}}

\newcommand {\bes} {\begin{description}}
\newcommand{\ens} {\end{description}}

\newcommand {\beq} {\begin{quote}}
\newcommand {\enq} {\end{quote}}
\newcommand {\bit} {\begin{itemize}}
\newcommand {\enit} {\end{itemize}}

\newenvironment{pozz}{\color{black}}{\color{black}}

\newenvironment{proof}{\noindent {\bf Proof.}} {}

\newtheorem{theorem}{Theorem}
\newtheorem{proposition}{Proposition}
\newtheorem{definition}{Definition}
\newtheorem{example}{Example}

\documentclass[letterpaper]{article}
\usepackage{aaai16}

\usepackage{times}
\usepackage{helvet}
\usepackage{courier}
\usepackage{times}
\usepackage{undertilde}
\usepackage{graphicx}
\usepackage{latexsym}
\usepackage{graphicx}
\usepackage{color}
\usepackage{amssymb}
\usepackage{colortbl}
\usepackage{amsmath}
\usepackage{microtype}

 \title{A strengthening of rational closure in DLs: \\
 reasoning about multiple aspects}

\author{Valentina Gliozzi \\ Center for Logic, Language and Cognition\\ Dipartimento di Informatica - 
Universit\`a di Torino - Italy\\ valentina.gliozzi@unito.it}

\begin{document}

 \maketitle
 
 \section{Abstract}
 

We propose a logical analysis of the concept of typicality, central in human cognition  (Rosch,1978). 
We start from a previously proposed extension of the basic Description Logic $\alc$ 
with a typicality operator $\tip$ that allows to consistently represent the attribution to classes of individuals of properties  with exceptions (as in the classic example (i)typical birds  fly, (ii) penguins are birds but (iii)typical penguins don't fly). We then strengthen this extension in order  to separately reason about the typicality with respect to different aspects (e.g., flying, having nice feather: in the previous example, penguins may not inherit the property of flying, for which they are exceptional, but can nonetheless inherit other properties, such as  having nice feather). 

\section{Introduction}
In \cite{AIJ15} it is proposed  a rational closure strengthening of $\alc$.
This strengthening allows to perform non monotonic reasoning in $\alc$ in a computationally efficient way.
The extension, as already the related logic $\alctmin$ proposed in \cite{AIJ13}  and the weaker (monotonic) logic $\alct$ presented in \cite{FI09}, allows to consistently represent typical properties with exceptions that could not be  represented in standard $\alc$.

\noindent For instance, in all the above logics one can say that:

{\bf  \small SET 1: 

Typical students don't earn money

Typical working students do earn money

Typical apprentice working students  don't earn money}

\noindent without having to conclude that there cannot exist working students nor apprentice working students. On the contrary, in standard $\alc$ typicality cannot be represented, and these three  propositions  can only be expressed by the stronger ones:

{ \bf  \small SET 2:

 Students don't earn money {\tiny (Student $\sqsubseteq \neg$ EarnMoney)}

Working students do earn money {\tiny(Worker $\sqcap$ Student $\sqsubseteq$ EarnMoney)}

Apprentice working students  don't earn money {\tiny (Worker $\sqcap$ Apprentice $\sqcap$ Student $\sqsubseteq \neg$ EarnMoney)}}

\noindent  These propositions are consistent in $\alc$ only if there are no working students nor apprentice working students.

In all the extensions of $\alc$ mentioned above one can represent the  set of propositions in $ SET 1$ by means of a typicality operator $\tip$ that, given a concept $C$ (e.g. Student) singles out the most typical instances of $C$: so, for instance,  $\tip(Student)$ refers to the typical instances of the concept Student. The semantics of $\tip$ is given by means of a preference relation $<$ that compares the typicality of two individuals: for any two $x$ and $y$, $x < y$ means that $x$ is more typical than $y$. Typical instances of a concept $C$ are those minimal with respect to $<$ (formally, as we will see later, $(\tip(C))^I = min_<(C)^I$, where $min_<(C)^I = \{x \in C^I: \not\exists y \in C^I$ s.t. $y< x \}$).

The operator $\tip$ has all the properties that, in the analysis of Kraus Lehmann and Magidor \cite{KrausLehmannMagidor:90} any non monotonic entailment should have. For instance, $\tip$ satisfies the principle of cautious monotonicity, according to which if $\tip(Student) \sqsubseteq Young$, then  $\tip(Student) = \tip(Student \sqcap Young)$).
The precise relations  between the properties of $\tip$ and preferential entailment are established in \cite{FI09}. 

Although the extensions of $\alc$ with the typicality operator $\tip$ allow to express $SET 1$ of propositions, the resulting logic is monotonic, and it does not allow to perform some wanted, non monotonic inferences. For instance, it does not allow to deal with irrelevance which is the principle that  from the fact that typical students are young, one would want to derive that typical blond students also are young, since being blond is irrelevant with respect to youth. As another example,
when knowing that an individual, say John, is a student, and given $SET 1$ of propositions, one would want to conclude that John is a typical student and therefore does not earn money. On the other hand, when knowing that John is a working student, one would want to conclude that he is a typical working student and therefore does earn money. In other words one would want to assume that an individual is a typical instance of the most specific class it belongs to, in the absence of information to the contrary.

These stronger inferences all hold in the strengthening of $\alct$ presented in  \cite{AIJ13,AIJ15}. 
In particular, \cite{AIJ15} proposes an adaptation to $\alc$ of the well known mechanism of {\em rational closure}, first proposed by Lehman and Magidor in \cite{whatdoes}. From a semantic point of view, this strengthening of $\alct$ corresponds to restricting one's attention to minimal models, that minimize the height (rank) of all domain elements with respect to $<$ (i.e. that minimize the length of the $<$-chains starting from all individuals).
Under the condition that the models considered are canonical, the semantic characterization corresponds to the syntactical rational closure. This semantics supports all the above wanted inferences, and the nice computational properties of  rational closure guarantee that whether the above inferences are valid or not can be computed in reasonable time.

The main drawback of rational closure is that it is an {\bf all-or-nothing} mechanism: for any subclass $C'$ of $C$ it holds that
either the typical members of $C'$ inherit all the properties of $C$ or they don't inherit any property. Once the typical members of $C'$ are recognized as exceptional with respect to $C$ for a given aspect, they become exceptional for all aspects. Consider the classic birds/penguins example, expressed by propositions: 

{\bf \small SET 3:

Typical birds have nice feather

Typical birds fly

Penguins are birds

Typical penguins do not fly}

\noindent In this case, since penguins are exceptional with respect to the aspect of flying, they are {\em non-typical} birds, and for this reason they do not inherit any of the typical properties of birds.
%
%
%

On the contrary, given $SET 3$ of propositions, one wants to conclude that:
\begin{itemize}
\item {\bf (**) Typical penguins have nice feather}
\end{itemize}
This is to say that one wants to separately reason about the different aspects: the property of flying is not related to the property of having nice feather, hence we want to separately reason on the two aspects.
%
%
%
%

\hide{In (CITA: Giordano, Gliozzi, Olivetti, Pozzato 2009, AIJ2013, DL2013), we have extended the Description Logic $\alc$ with a typicality operator $\tip$, and we have analyzed its rational properties (which are strongly related to the core properties of  non monotonic entailment studied by \cite{klm,whatdoes}).
The semantics of $\tip$ uses a preference relation $<$ that compares the typicality of individuals: for all concepts C, the typical instances of C ($\tip(C)$) are the instances of $C$ minimal w.r.t. $<$. As shown in (Giordano et al., DL2013), when considering only minimal, canonical such models (in which all possible combinations of concepts are represented and in which every individual is assumed to be as typical as possible), we obtain a logic which is an extension to DLs of the well-known mechanism of rational closure, that allows to make the following inferences. 

\begin{itemize}
\item[(i)]Typical birds fly
\item[(ii)]Typical penguins don't fly
\item[(iii)] penguins are birds
\end{itemize}
and 

\begin{itemize}
\item[(iv)]Tweety is a bird
\item[(v)]Pio is a penguin
\end{itemize}

the logic allows to infer that

\begin{itemize}
\item[(vi)]Tweety flies
\item[(vii)]Pio does not fly
\item[(viii)]Typical black bird fly
\end{itemize}

(viii) is the well-known principle of irrelevance, (vi)(vii) show that our logic allows to infer that all individuals are typical instances of the most specific concept they belong to. Notice that (vi)(vii)(viii) are not supported by  standard, classical DLs such as $\alc$, which cannot even consistently represent (i)-(v).

The weakness of this logic (as of propositional rational closure \cite{whatdoes}) is that it is an all-or-nothing approach: if a subclass is exceptional for a given aspect (in the example, for flying), it does not inherit any of the properties of the more general class. In the example, if we further knew that 
\begin{itemize}
\item[(ix)]typical birds have nice feather 
\end{itemize}

we could not conclude that

\begin{itemize}
\item[(x)] typical penguins have nice feather
\end{itemize}

Indeed, since penguins are birds that do not fly, they are {\em non-typical} birds, and for this reason they do not inherit any of the typical properties of birds.  The above logic does not allow to separately reason about different aspects.

We here propose a strengthening of the above logic by introducing several preference relations to separately reason about distinct aspect: a preference relation $<_{Fly}$ will be used to express typicality with respect to flying, a distinct  $<_{HasNiceFeather}$ will be used to express typicality with respect to the property of having nice feather. We show that the resulting approach is a strengthening of the approach previously presented that allows to separately reason on different aspects and, for instance, allows to infer (x) from (i)-(ix).

} 
Here we propose a strengthening of the  semantics used for rational closure in $\alc$ \cite{AIJ15} that only used a single preference relation $<$ by allowing, beside $<$, several preference relations that compare the typicality of individuals with respect to a given aspect. {\bf Obtaining a strengthening of rational closure is the purpose of this work. This puts strong constraints on the resulting semantics, and defines  the horizon of this work}.
In this new semantics we can express the fact that, for instance, $x$ is more typical than $y$ with respect to the property of flying but $y$ is more typical that $x$ with respect to some other property, as the property of having nice feather. To this purpose we consider preference relations indexed by concepts that stand for the above mentioned aspects under which we compare individuals. So we will write that  $x   <_A y$ to mean that $x$ is preferred to $y$ for what concerns aspect $A$: for instance $x <_{Fly} y$ means that $x$ is more typical than $y$ with respect to the property of flying.

\hide{Starting from the various $<_A$:
\begin{itemize}
\item starting from the various preference relations $<_A$ we can define a unique preference relation $<$ satisfying the properties of $<$ in \cite{AIJ15}, DL (?). NO: SI OTTIENE QUALCOSA DI PIU FORTE
\item when restricting attention to minimal canonical models, this semantics proves to be stronger than rational closure
\end{itemize}
}

We therefore proceed as follows: we first recall the semantics of the extension of $\alc$ with a typicality operator which was at the basis of the definition of rational closure and semantics in 
\cite{dl2013,AIJ15}. We then expand this semantics by introducing several preference relations, that we then minimize obtaining our new minimal models' mechanism. As we will see this new semantics leads to a strengthening of rational closure, allowing to separately reason about the inheritance of different properties.


\section{The  operator $\tip$ and the General Semantics}\label{sez:semantica}
Let us briefly recall 
the logic $\alctr$  which is at the basis of a rational closure construction  proposed in \cite{AIJ15} for $\alc$.
The intuitive idea of $\alctr$ is to extend the standard $\alc$ with concepts of the form $\tip(C)$,  whose intuitive meaning is that
$\tip(C)$ selects the {\em typical} instances of a concept $C$, to distinguish between the properties that
hold for all instances of concept $C$ ($C \sqsubseteq D$), and those that only hold for the typical
such instances ($\tip(C) \sqsubseteq D$).  The $\alctr$ language is defined as follows:
\hide{Formally, the language is defined as follows. }
 $C_R:= A \tc \top \tc \bot \tc  \nott C_R \tc C_R \sqcap C_R \tc C_R \sqcup C_R \tc \forall R.C_R \tc \exists R.C_R$, and
   $C_L:= C_R \tc  \tip(C_R)$, where $A$ is a concept name and $R$ a role name.
    A KB is a pair (TBox, ABox). TBox contains a finite set
of  concept inclusions  $C_L \sqsubseteq C_R$. ABox
contains a finite set of assertions of the form $C_L(a)$ and $R(a,b)$, where $a, b$ are individual constants.

\noindent The semantics of $\alctr$  is defined
in terms of rational
models:
 ordinary models of $\alc$ are equipped with a \emph{preference relation} $<$ on
the domain, whose intuitive meaning is to compare the ``typicality''
of domain elements: $x < y$ means that $x$ is more typical than
$y$. Typical members of a concept $C$, instances of
$\tip(C)$, are the members $x$ of $C$ that are minimal with respect
to $<$ (such that there is no other member of $C$
more typical than $x$). In rational models $<$ is further assumed to be modular: for all $x, y, z \in \Delta$, if
$x < y$ then either $x < z$ or $z < y$. These rational models characterize
$\alctr$.

\vspace{-0.15cm}
\begin{definition}[Semantics of $\alctr$ \cite{AIJ15}]\label{semalctr} A model $\emme$ of $\alctr$ is any
structure $\langle \Delta, <, I \rangle$ where: $\Delta$ is the
domain;   $<$ is an irreflexive, transitive, and modular  relation over
$\Delta$ that satisfies the \emph{finite chain condition}
(there is no infinite $<$-descending chain, hence if $S \neq \emptyset$, also $min_<(S) \neq \emptyset$); $I$ is the extension function that maps each
concept name $C$ to $C^I \subseteq \Delta$, each role name $R$
to  $R^I \subseteq \Delta^I \times \Delta^I$
and each individual constant $a \in \mathcal{O}$ to $a^I \in \Delta$.
For concepts of
$\alc$, $C^I$ is defined in the usual way. For the $\tip$ operator, we have
$(\tip(C))^I = min_<(C^I)$.
\end{definition}
\vspace{-0.2cm}

As shown in \cite{AIJ15}, the logic $\alctr$ enjoys the  finite model property and finite $\alctr$ models can be equivalently defined by postulating the existence of
a function $k_{\emme}: \Delta \longmapsto \mathbb{N}$, where $k_{\emme}$ assigns a finite rank to each world: the rank $k_{\emme}$  of a domain element $x \in \Delta$ is the
length of the longest chain $x_0 < \dots < x$ from $x$
to a minimal $x_0$ (s. t. there is no ${x'}$ with  ${x'} < x_0$). The rank $k_\emme(C_R)$ of a concept $C_R$ in $\emme$ is $i = min\{k_\emme(x):
x \in C_R^I\}$.

A model $\emme$ satisfies a knowledge base $K$=(TBox,ABox) if it satisfies  its TBox (and for all  inclusions $C \sqsubseteq D$  in TBox, it holds $C^I \subseteq D^I$), 
and its ABox (for all $C(a)$  in ABox,  $a^I \in C^I$, and for all $aRb$ in
ABox,  $(a^I,b^I) \in R^I$).
%
%
%
%
%
 A query $F$ (either an assertion $C_L(a)$ or an inclusion relation $C_L \sqsubseteq C_R$) is logically (rationally) entailed by a knowledge base $K$ ($K \models_{\alctr} F$) if $F$ holds in all models satisfying $K$. 

Although the typicality operator $\tip$ itself  is nonmonotonic (i.e.
$\tip(C) \sqsubseteq D$ does not imply $\tip(C \sqcap E)
\sqsubseteq D$), the logic $\alctr$ is monotonic: what is logically entailed by $K$ is still entailed by any $K'$ with $K \subseteq K'$.

In \cite{dl2013,AIJ15} the non monotonic mechanism of rational closure has been defined over $\alctr$, which extends to DLs
the notion of rational closure proposed in the propositional context by Lehmann and Magidor \cite{whatdoes}.  The definition is based on the notion of exceptionality. Roughly speaking $\tip(C) \sqsubseteq D$ holds (is included in the rational closure) of $K$ if $C$ (indeed, $C \sqcap D$) is less exceptional than $C \sqcap \neg D$. We briefly recall this construction and we refer to \cite{dl2013,AIJ15} for full details.
Here we only consider rational closure of TBox, defined as follows.

\begin{definition}[Exceptionality of concepts and inclusions]\label{definition_exceptionality}
Let $T_B$ be a TBox and $C$ a concept. $C$ is
said to be {\em exceptional} for $T_B$ if and only if $T_B \models_{\alctr} \tip(\top) \sqsubseteq
\neg C$. A \tip-inclusion $\tip(C) \sqsubseteq D$ is exceptional for $T_B$ if $C$ is exceptional for $T_B$. The set of \tip-inclusions of $T_B$ which are exceptional in $T_B$ will be denoted
as $\mathcal{E}$$(T_B)$.
\end{definition}

\noindent Given a DL  TBox,
it is possible to define a sequence of non increasing subsets of
TBox ordered according to the exceptionality of the elements $E_0 \supseteq E_1, E_1 \supseteq E_2, \dots$ by letting $E_0 =\mbox{TBox}$ and, for
$i>0$, $E_i=\mathcal{E}$$(E_{i-1}) \unione \{ C \sqsubseteq D \in \mbox{TBox}$ s.t. $\tip$ does not occurr in $C\}$.
Observe that, being KB finite, there is
an $n\geq 0$ such that, for all $m> n, E_m = E_n$ or $E_m =\emptyset$.
A concept $C$ has {\em rank} $i$ (denoted $\rf(C)=i$) for TBox,
iff $i$ is the least natural number for which $C$ is
not exceptional for $E_{i}$. {If $C$ is exceptional for all
$E_{i}$ then $\rf(C)=\infty$ ($C$ has no rank).}

Rational closure builds on this notion of exceptionality:

\begin{definition}[Rational closure of TBox] \label{def:rational closureDL}
Let KB = (TBox, ABox) be a DL knowledge base. The
rational closure  of TBox 
    $\mbox{$\overline{\mathit{TBox}}$}=\{\tip(C) \sqsubseteq D \tc \mbox{either} \ \rf(C) < \rf(C \sqcap \nott D)$ $\mbox{or} \ \rf(C)=\infty\} \ \unione \
    \{C \sqsubseteq D \tc \ \mbox{KB} \ \models_{\alctr} C \sqsubseteq D\}$,
where $C$ and $D$ are $\alc$ concepts.
\end{definition}

As a very interesting property, in the context of DLs, the rational closure has a very interesting complexity:
deciding  if an inclusion $\tip(C) \sqsubseteq D$ belongs to the rational closure of TBox is a problem in \textsc{ExpTime} \cite{AIJ15}.

In \cite{AIJ15} it is shown that the semantics corresponding to rational closure can be given in terms of minimal canonical $\alctr$ models. With respect to standard $\alctr$ models, in these models the rank of each domain element is as low as possible (each domain element is assumed to be as typical as possible). 
This is expressed by the following definition.

 \begin{definition}[Minimal models of $K$ (with respect to $TBox$)]\label{Preference between models in case of fixed valuation} 
Given $\emme = $$\langle \Delta, <, I \rangle$ and $\emme' =
\langle \Delta', <', I' \rangle$ , we say that $\emme$ is preferred to
$\emme'$ \hide{with respect to the fixed interpretations minimal
semantics} ($\emme < \emme'$) if:
$\Delta = \Delta'$,
$C^I = C^{I'}$ for all concepts $C$,
for all $x \in \Delta$, it holds that $ k_{\emme}(x) \leq k_{\emme'}(x)$ whereas
there exists $y \in \Delta$ such that $ k_{\emme}(y) < k_{\emme'}(y)$.

Given a knowledge base $K = \langle TBox, ABox \rangle$, we say that
$\emme$ is a minimal model of $K$ (with respect to TBox)  if it is a model satisfying $K$ and  there is no
$\emme'$ model satisfying $K$ such that $\emme' < \emme$.
\end{definition}
Furthermore, the models corresponding to rational closure are canonical. This property, expressed by the following definition, is needed when reasoning about the (relative) rank of the concepts: it is important to have them all represented.

%

\begin{definition}[Canonical model\hide{with respect to $\lingconc$}]\label{def-canonical-model-DL}
Given $K$=(TBox,ABox),
a  model $\emme=$$\sx \Delta, <, I \dx$ satisfying $K$ is 
{\em canonical} 
 if for each set of concepts
$\{C_1, C_2, \dots, C_n\}$
consistent with $K$, there exists (at least) a domain element $x \in \Delta$ such that
$x \in (C_1 \sqcap C_2 \sqcap \dots \sqcap C_n)^I$. \end{definition}
%
%
%
\begin{definition}[Minimal canonical models (with respect to TBox)]\label{Preference between models wrt TBox}
$\emme$ is a canonical model of $K$ minimal with respect to TBox 
if it satisfies $K$, it is  minimal with respect to $TBox$ (Definition \ref{Preference between models in case of fixed
valuation}) and it is canonical (Definition \ref{def-canonical-model-DL}).
\end{definition}

The correspondence between minimal canonical models and rational closure is established by the following key theorem. 

\begin{theorem}[\cite{AIJ15}]\label{Theorem_RC_TBox}
Let $K$=(TBox,ABox) be a knowledge base and $C \sqsubseteq D$ a query.
We have that $C \sqsubseteq D \in$ $\overline{\mathit{TBox}}$ if and only if $C \sqsubseteq D$ holds in all minimal  canonical models 
of $K$ with respect to TBox (Definition \ref{Preference between models wrt TBox}).
\end{theorem}

\hide{In \cite{AIJ15} the notion of rational closure is extended in order to deal with ABox. [MENZIONA ALGORITMO]
The semantic notions just provided are extended as follows in order to deal with ABox.
\begin{definition}[Canonical model of $K$ minimally satisfying ABox]\label{model-minimally-satisfying-Abox}
Given \\ $K$=(TBox,ABox), let $\emme = $$\langle \Delta, <, I \rangle$ and $\emme' =
\langle \Delta', <', I' \rangle$ be two canonical models of $K$ which are minimal with respect to TBox (Definition \ref{Preference between models wrt TBox}). We say that $\emme$ is preferred to $\emme'$ with respect to ABox, and we write $\emme <_{\mathit{ABox}} \emme'$, if, for all individual constants $a$ occurring in ABox, it holds that $k_{\emme}(a^I) \leq k_{\emme'}(a^{I'})$ and there is at least one individual constant $b$ occurring in ABox such that  $k_{\emme}(b^I) < k_{\emme'}(b^{I'})$.
\end{definition}

The following theorem (Theorem 12 in \cite{AIJ15})  shows that these last models correctly capture rational closure applied to ABox. 
\begin{theorem}[Semantic characterization of $\overline{\mathit{ABox}}$]\label{Semantic-rational-closure-Abox}
Given $K$=(TBox, ABox), for all  individual constant $a$ in ABox, we have that $C(a) \in$ $\overline{\mathit{ABox}}$
just in case $C(a)$ holds in all minimal canonical models of $K$ minimally satisfying ABox (Definition \ref{model-minimally-satisfying-Abox}).
\end{theorem}
} 

\section{Semantics with several preference relations}

The main weakness of rational closure, despite its power and its nice computational properties, is that it is an all-or-nothing mechanism that does not allow to separately reason on single aspects. To overcome this difficulty, we here consider models with  several preference relations, one for each aspect we want to reason about. We assume this is any concept occurring in K:
we call ${\cal L_{A}}$ the set of these aspects (observe that $A$ may be non-atomic). 
For each aspect $A$, $<_A$ expresses the preference for aspect $A$ : $<_{Fly}$ expresses the preference for flying, so if we know that $\tip(Bird) \sqsubseteq Fly$, birds that do fly will be preferred  to birds that do not fly, with respect to aspect fly, i.e. with respect to $<_{Fly}$. All these preferences, as well as the global preference relation $<$, satisfy the properties in Definition \ref{def-enrichedmodelR} below.  
We now enrich the definition of an $\alctr$ model given above (Definition \ref{semalctr}) by taking into account preferences with respect to all of the aspects. In the semantics we can express that for instance $x <_{A_i} y$, whereas $y <_{A_j} x$ ($x$ is preferred to $y$ for aspect $A_i$ but $y$ is preferred to $x$ for aspect $A_j$). 

This semantic richness allows to obtain a strengthening of rational closure in which typicality with respect to every aspect is maximized.
Since we want to compare our approach to rational closure,  we keep the language the same than in $\alctr$. In particular, we only have one single  typicality operator $\tip$. However, the semantic richness could motivate the introduction of  several typicality  operators $\tip_{A_1} \dots \tip_{A_n}$ by which one might want to explicitly talk in the language about the typicality w.r.t. aspect $A_1$, or $A_2$, and so on. We leave this extension for future work. 

\begin{definition}[Enriched rational models]\label{def-enrichedmodelR}
Given a knowledge base K, we call an enriched rational model  a structure $\emme = \langle \Delta, <, <_{A_1}, \dots ,<_{A_n}, I \rangle$, where
$\Delta$, $I$ are defined as in Definition \ref{semalctr}, and $<, <_{A_1}, \dots ,<_{A_n}$ are preference relations over $\Delta$, with the properties of being irreflexive, transitive, satisfying the finite chain condition, modular (for all $x, y, z \in \Delta$, if $x <_{A_i} y$ then either $x <_{A_i} z$ or $z <_{A_i} y$). 

For  all $<_{A_i}$ and for $<$  it holds that $min_{<_{A_i}}(S) = \{x \in S$ s.t. there is no $x_1 \in S$ s.t. $x_1 <_{A_i} x\}$ 
and $min_<(S) = \{x \in S$ s.t. there is no $x_1 \in S$ s.t. $x_1 < x\}$ and $(\tip(C))^I = min_<(C^I)$.

$<$ satisfies the further conditions that $x < y$ if:\\
\noindent (a) there is $A_i$ such that $x <_{A_i} y$, and there is no $A_j$ such that $y <_{A_j} x$ or;\\
\noindent (b) there is $\tip(C_i) \sqsubseteq A_i \in K$ s.t. $ y \in (C_i \sqcap \neg A_i)^I$, and for all $\tip(C_j) \sqsubseteq A_j \in K$ s.t.  $x \in (C_j \sqcap \neg A_j)^I$ , there is $\tip(C_k) \sqsubseteq A_k \in K$ s.t.  $y \in (C_k \sqcap \neg A_k)^I$ and $k_{\emme}(C_j) < k_{\emme}(C_k)$.
\end{definition}

In this semantics the global preference relation $<$ is related to  the various preference relations $<_{A_i}$ relative to single aspects $A_i$. Given (a)  $x < y$ when $x$ is preferred to $y$ for a single aspect $A_i$, and there is no aspect $A_j$ for which $y$ is preferred to $x$. (b) captures the idea that in case two individuals are preferred with respect to different aspects, preference (for the global preference relation) is given to the individual that satisfies all typical properties of the {\bf most specific} concept (if $C_k$ is more specific than $C_j$, then  $k_{\emme}(C_j) < k_{\emme}(C_k)$), as illustrated by Example \ref{example-global-relation} below. 

We insist in highlighting that this semantics  somewhat complicated is needed since we want to provide a strengthening of rational closure. For this, we have to respect the constraints imposed by rational closure. One might think in the future to study a semantics in which only (a) holds.
We have not considered such a simpler semantics since it  would no longer be a strengthening of the semantics corresponding to rational closure, and is therefore out of the focus of this work.

In order to be a model of $K$ an $\alctre$ model must satisfy the following constraints.
\begin{definition}[Enriched rational models of K]\label{enriched-model-of-K}
Given a knowledge base K, and an enriched rational model for $K$ $\emme = \langle \Delta, <, <_{A_1}, \dots ,<_{A_n}, I \rangle$,
$\emme$ is a model of $K$ if it satisfies both its TBox
and its ABox, where $\emme$  satisfies TBox if for all  inclusions $C \sqsubseteq A_i \in TBox$: 
if $\tip$ does not occur in $C$, then $C^I \subseteq {A_i}^I$
 if $\tip$ occurs in $C$, and $C$  is $\tip(C')$, then both
(i)  $min_<({C'}^I) \subseteq {A_i}^I$ and (ii) $min_{<_{A_i}}({C'}^I) \subseteq {A_i}^I$. \\
$\emme$ satisfies ABox  if
(i) for all $C(a)$  in ABox,  $a^I \in C^I$, (ii) for all $aRb$ in
ABox,  $(a^I,b^I) \in R^I$
\end{definition}


\begin{example}\label{example-global-relation}
Let 
 {\small $K = \{Penguin \sqsubseteq Bird,  \tip(Bird) \sqsubseteq HasNiceFeather$, $\tip(Bird) \sqsubseteq Fly$,  $\tip(Penguin) \sqsubseteq \neg Fly \}$.
$\elle_A = \{HasNiceFeather, Fly,  \neg Fly, Bird, Penguin\}$}.  
We consider an $\alctre$ model $\emme$ of $K$, that we don't fully describe but which we only use to observe the behavior of two Penguins 
$x$, $y$ with respect to the properties of (not) flying and having nice feather.
 In particular, let us consider the three preference relations: $<, <_{\neg Fly}, <_{HasNiceFeather}$. 

Suppose  $x <_{\neg Fly} y$ (because $x$, as all typical penguins, does not fly whereas $y$ exceptionally does)  and there is no other aspect $A_i$ such that $y <_{A_i} x$, and in particular it does not hold that  $y <_{HasNiceFeather} x$ (because for instance  both have a nice  feather). In this case, obviously it holds that $x < y$ (since (a) is satisfied). 

Consider now a more tricky situation in which again $x <_{\neg Fly} y$ holds (because for instance  $x$ does not fly whereas $y$ flies), 
($x$ is a typical penguin for what concerns Flying) but  this time $y <_{HasNiceFeather} x$ holds (because for instance $y$  has a nice feather, whereas $x$ has not).
So $x$ is preferred to $y$ for a given aspect whereas $y$ is preferred to $x$ for another aspect. 
However, $x$ enjoys the typical properties of penguins, and violates the typical properties of birds, whereas $y$ enjoys the typical properties of birds and violates those of penguins. Being concept Penguin more specific than concept Bird, we prefer $x$ to $y$, since we  prefer the individuals that inherit the properties of the most specific concepts of which they are instances.
This is exactly what we get: by (b)   $x < y$ holds.
\end{example}
Logical entailment for $\alctre$  is defined as usual: a query (with form $C_L(a)$ or $C_L \sqsubseteq C_R$) is logically entailed by $K$ if it holds in all models of $K$, as stated by the following definition.
%
%
%
The following theorem shows the relations between $\alctre$ and $\alctr$. Proofs are omitted due to space limitations.
 \begin{theorem}\label{equivalence-multiple}
If $K \models_{\alctr} F$ then also $K \models_{\alctre} F$. If $\tip$ does not occur in $F$ the other direction also holds: If $K \models_{\alctre} F$ then also $K \models_{\alctr} F$.
 \end{theorem} 
 The following example shows that $\alctre$ alone is not strong enough, and this motivates the minimal models' mechanism that we introduce in the next section. In the example we show that $\alctre$ alone does not allow us to perform the stronger inferences with respect to rational closure mentioned in the Introduction (and in particular, it does not allow to infer (**), that typical penguins have a nice feather). 

\begin{example}\label{esempio-alctre}
Consider 
the above Example \ref{example-global-relation}.
As said in the Introduction, in rational closure we are not able to reason separately about 
the property of flying or not flying, and the property of having or not having a nice feather. Since penguins are exceptional birds  with respect to the property of flying, in rational closure which is an all-or-nothing mechanism, they do not inherit any of the properties of typical birds. In particular, they do not inherit the property of having a nice feather, even if this property and the fact of flying are  independent from each other and there is no reason why being exceptional with respect to one property should block the inheritance of the other one. Does our enriched semantics enforce the separate inheritance of independent properties? 

Consider a model $\emme$ in which we have $\Delta = \{x, y, z\}$, where $x$ is a bird (not a penguin) that flies and has a nice feather ($x \in Bird^I, x \in Fly^I, x \in HasNiceFeather^I, x \not \in Penguin^I$), $y$ is a penguin that does not fly and has a nice feather ($y \in Penguin^I, y \in Bird^I, y \not \in Fly^I, y \in HasNiceFeather^I$), $z$ is a penguin that does not fly and has no nice feather  ($z \in Penguin^I, z \in Bird^I, z \not \in Fly^I, z \not \in HasNiceFeather^I$). 
Suppose it holds that
$x <_{Fly} y$, $x <_{Fly} z$,
$x <_{HasNiceFeather} z$, $y <_{HasNiceFeather} z$, and  $x < y$, $x < z$, $y < z$. 
It can be verified that this is an $\alctre$ model, satisfying $\tip(Penguin) \sqsubseteq HasNiceFeather$ (since the only typical Penguin is $y$, instance of HasNiceFeather). 

Unfortunately, this is not the only $\alctre$ model of $K$. For instance there can be $\emme'$ equal to $\emme$ except from the fact that $y <_{HasNiceFeather} z$ 
does not hold, nor $y < z$ holds. It can be easily verified that this is also an $\alctre$ model of $K$  in which $\tip(Penguin) \sqsubseteq HasNiceFeather$ does not hold (since now also $z$ is a typical Penguin, and $z$ is not an instance of HasNiceFeather).
\end{example}

This example shows that although there are $\alctre$ models satisfying well suited  inclusions, the logic is not strong enough to limit our attention to these models. We would like to constrain our logic in order to exclude models like $\emme'$. Roughly speaking, we want to eliminate $\emme'$ because it is not minimal: although the model as it is satisfies $K$, so $y$ does not {\em need} to be preferred to $z$ to satisfy $K$ (neither with respect to $<$ nor with respect to $<_{HasNiceFeather}$), intuitively we would like to prefer $y$ to $z$ (with respect to the property HasNiceFeather, whence in this case with respect to the global $<$), since $y$ does not falsify any of the inclusions with HasNiceFeather, whereas z does. This is obtained by imposing the constraint of considering only models minimal with respect to all relations $<_A$, defined as in Definition \ref{minimal_enriched_models}  below. Notice that the wanted inference does not hold in $\alctr$ minimal canonical models corresponding to rational closure: in these models $y < z$ does never hold (the two elements have the same rank) and this semantics does not allow us to prefer $y$ to $z$. 
By adopting the restriction to minimal canonical models, 
we obtain a semantics which is stronger than rational closure (and therefore enforces all conclusions enforced by rational closure) and, furthermore, separately allows to reason on different aspects.

Before we end the section, similarly to what done above, let us introduce a rank of a domain element with respect to an aspect. We will use this notion in the following section.
\begin{definition}\label{definition_height}
The rank $k_{{A_i}_\emme} (x)$  of a domain element $x$ with respect to $<_{A_i}$ in $\emme$ is the
length of the longest chain $x_0 <_{A_i} \dots <_{A_i} x$ from $x$
to a minimal $x_0$ (s.t. for no ${x'}$  ${x'} <_{A_i} x_0$). 
To refer to the rank of an element $x$ with respect to the preference relation $<$ we will simply write 
$k_\emme (x)$. 
\end{definition}
The notion just introduced will  be useful in the following. Since   $k_{{A_i}_\emme} $ and $<_{A_i}$ are clearly interdefinable (by the previous definition and by the properties of $<_{A_i}$ it easily follows that in all enriched models $\emme$, $x <_{A_i} y$ iff $k_{{A_i}_\emme}(x) <  k_{{A_i}_\emme}(y)$, and $x < y$ iff $k_{\emme}(x) <  k_{\emme}(y)$ ), we will shift from one to other whenever this simplifies the exposition. 

\section{Nonmonotonicity and relation with rational closure}\label{section-relations}

We here define a minimal models mechanism starting from the enriched models of the previous section. With respect to the minimal canonical models used in \cite{AIJ15} we define minimal models by separately minimizing all the preference relations with respect to all aspects (steps (i) and (ii) in the definition below), before minimizing $<$ (steps (iii) and (iv) in the definition below). 
By the constraints linking $<$ to the preference relations $<_{A_1} \dots <_{A_n}$, this leads to preferring (with respect to the global $<$) the individuals that are minimal with respect to all $<_{A_i}$ for all aspects $A_i$, or to aspects of most specific categories than of more general ones. It turns out that this leads to a stronger semantics than what is obtained by directly minimizing $<$.
%
\begin{definition}[Minimal Enriched Models]\label{minimal_enriched_models} 
Given two $\alctre$ enriched models $\emme = \langle \Delta, <_{A_1}, \dots ,<_{A_n}, <, I  \rangle$ and $\emme' =
\langle \Delta',  <'_{A_1}, \dots ,<'_{A_n}, <', I' \rangle$ we say that $\emme'$ is preferred to
$\emme$ with respect to the single aspects (and write $\emme' <_{Enriched_Aspects} \emme$) if $\Delta = \Delta'$, $I =
I'$, and:
\begin{itemize}
\item  (i) for all $x \in \Delta$, for all $A_i$: $ k_{{A_i}_{\emme'}}(x) \leq k_{{A_i}_{\emme}}(x)$;
\item (ii) for some $y \in \Delta$, for some $A_j$, $ k_{{A_j}_{\emme'}}(y) < k_{{A_j}_{\emme}}(y)$  
\end{itemize} 
We let the set $Min_{Aspects} = \{\emme:$ there is no $\emme'$ such that $\emme' <_{Enriched_Aspects} \emme\}$.\\
Given $\emme$ and $\emme' \in Min_{Aspects}$, we 
 say that  $\emme'$ is overall preferred to $\emme$ (and write $\emme' <_{Enriched} \emme$)
 if$\Delta = \Delta'$, $I =
I'$, and:
 \begin{itemize}
\item  (iii) for all $x \in \Delta$, $ k_{{\emme'}}(x) \leq k_{{\emme}}(x)$;
\item (iv) for some $y \in \Delta$ ,  $ k_{{\emme'}}(y) < k_{{\emme}}(y)$ 
\end{itemize}  

We call $\emme$ a {\em minimal enriched model} of $K$ if it is a model of $K$ and there is no $\emme'$ model of $K$ such that  {\small $ \emme' <_{Enriched} \emme$}.
 \end{definition}
$K$ minimally entails a query $F$  if $F$ holds in all minimal $\alctre$ 
models of  $K$. We write {\small $K \models_{{\alctre}_{\mathit{min}}} F$}.
%
%
%
%
%
%
We have developed the semantics above in order to overcome a weakness of rational closure, namely its all-or-nothing character. 
In order to show that the semantics hits the point, we show here that the semantics is stronger than the one corresponding to rational closure. 
Furthermore, Example \ref{esempio-strengthening of rational closure} below shows that indeed we have strengthened rational closure 
by making it possible to separately reason on the different properties.  Since the semantic characterization of rational closure is given in terms of rational {\em canonical} models, 
here we restrict our attention to enriched rational models which are {\em canonical}. 
\begin{definition}[Minimal canonical enriched models of K]\label{Def-minimal-enriched-canonical}
An $\alctre$ enriched model $\emme$ is a minimal canonical enriched model of $K$ 
if it satisfies $K$, it is  minimal (with respect to Definition \ref{minimal_enriched_models})
and it is canonical: for all the sets of concepts $\{C_1, C_2, \dots,$ $ C_n\}$\hide{ \subseteq \lingconc$} s.t. $K \not\models_{\alctre} C_1 \sqcap C_2 \sqcap \dots \sqcap C_n \sqsubseteq \bot$, there exists (at least) a domain element $x$ such that
$x \in (C_1 \sqcap C_2 \sqcap \dots \sqcap C_n)^I$. 
\end{definition}
We call $\alctre + {{min-canonical}}$ the semantics obtained by restricting attention to minimal canonical enriched models. In the following we will write:\\ $K \models_{\alctre + {{min-canonical}}} C \sqsubseteq D$ to mean that $C \sqsubseteq D$ holds in all minimal canonical enriched models of $K$.
The following example shows that this semantics allows us to correctly deal with the wanted inferences of the Introduction, as (**). The fact that the semantics $\alctre + {{min-canonical}}$ is a genuine strengthening of the semantics corresponding to rational closure is formally shown in Theorem \ref{stronger_semantics} below.
\begin{example}\label{esempio-strengthening of rational closure}
Consider any minimal canonical model $\emme^*$ of the same $K$ used in Example \ref{example-global-relation}.
It can be easily verified that in $\emme^*$ there is a domain element $y$ which is a penguin that does not fly and has a nice feather ($y \in Penguin^I, y \in Bird^I, y \in HasNiceFeather^I$).
First, it can be verified that $y \in min_<(Penguin^I)$ (by Definition \ref{def-enrichedmodelR},  and since by minimality of $< _{Fly}$ and $< _{HasNiceFeather}$, $y \in min_{<_{Fly}}(Penguin^I)$ and $y \in min_{<_{HasNiceFeather}}(Penguin^I)$). Furthermore,  for all penguin $z$ that has not a nice feather, $y < z$ (again by Definition \ref{def-enrichedmodelR},  and since by minimality of $< _{Fly}$ and $< _{HasNiceFeather}$, $y <_{HasNiceFeather} z$).  From this, in all minimal canonical $\alctre$ models of $K$ it holds that $\tip(Penguin) \sqsubseteq HasNiceFeather$, i.e., $K \models_{\alctre + {{min-canonical}}}  \tip(Penguin) \sqsubseteq HasNiceFeather$, which was the wanted inference (**) of the Introduction.
\end{example}




{\bf The following theorem is the important technical result of the paper:} 

\begin{theorem}\label{stronger_semantics}{\bf The minimal models semantics $\alctre + {min-canonical}$ is stronger than the semantics for rational closure.
Let ($K=TBox, ABox$). If $C \sqsubseteq D \in \overline{TBox}$ then $K {\small \models_{\alctre + {{min-canonical}}}} C \sqsubseteq D$.}
 \end{theorem} 
 \begin{proof}{\bf (Sketch)}
%
%
By contraposition suppose that $K \not\models_{\alctre + {{min-canonical}}} C \sqsubseteq D$. 
Then there is a minimal canonical enriched $\alctre$ model  $\emme =  \langle \Delta, <_{A_1}, \dots ,<_{A_n}, <, I  \rangle$  of $K$ and an $y \in C^I$ such that $y \not\in D^I$. 
All consistent sets of concepts consistent with $K$ w.r.t. $\alctre$ are also consistent with $K$ with respect to $\alctr$, and viceversa (by Theorem \ref{equivalence-multiple}).
By definition of {\em canonical}, there is also a canonical $\alctr$ model of $K$ $\emme_{RC} =  \langle {\Delta},  <_RC, I  \rangle$ be this model. If $C$ does not contain the $\tip$ operator, we are done: in $\emme_{RC}$, as in $\emme$, there is $y \in C^{I}$ such that $y \not\in D^{I}$, hence $C \sqsubseteq D$ does not hold in  $\emme_{RC}$, and $C \sqsubseteq D \not\in \overline{TBox}$.
If $\tip$ occurs in $C$, and $C = \tip(C')$, we still need to show that also in $\emme_{RC}$, as in $\emme$, 
$y \in min_{<_{RC}}({C'}^I)$.  We prove this by showing that for all $x,y \in \Delta$ if $x <_{RC} y$ in $\emme_{RC}$, then also 
$x <y$ in $\emme$. 
%
%
The proof is by induction on  $k_{\emme_{RC}}(x)$.
 
(a): let $k_{\emme_{RC}}(x)= 0$ and $k_{\emme_{RC}}(y)>0$. Since $x$ does not violate any inclusion, also in $\emme$ (by minimality of $\emme$) for all preference relations $<_{A_j}$  $k_{{A_j}_\emme}(x) = 0 $, and also  $k_{\emme}(x)= 0$. This cannot hold for $y$, for which $k_{\emme}(y)> 0$ (otherwise $\emme$ would violate 
$K$, against the hypothesis). Hence $x < y$ in $\emme$.

(b): let $k_{\emme_{RC}}(x)= i < k_{\emme_{RC}}(y)$, i.e. $x<_{RC} y$.
As $x <_{RC} y$ in $\emme_{RC}$ and the rank of $x$ in $\emme_{RC}$ is $i$,
there must be a $\tip(B_i) \sqsubseteq A_i \in E_i - E_{i+1}$ such that $x \in (\neg B_i \sqcup A_i)^I$ whereas $y \in (B_i \sqcap \neg A_i)^I$ in $\emme_{RC}$. Before we proceed let us notice that by definition of $E_i$, as well as by what stated just above on the relation between rank of a concept and $k_{\emme_{RC}}$, $k_{\emme_{RC}}(B_i)= k_{\emme_{RC}}(x)$ . We will use this fact below.
%
%
We show that, for any inclusion $\tip(B_l) \sqsubseteq A_l \in K$ that is violated by $x$, it holds that $k_{\emme}(B_l) < k_{\emme}(B_i)$,
so that, by (b), $x<y$.

Let $\tip(B_l) \sqsubseteq A_l \in K$ violated by $x$, i.e. such that  $x \in (B_l \sqcap \neg A_l)^I$.
Since $\emme_{RC}$ satisfies $K$,  there must be $x' <_{RC} x$ in $\emme_{RC}$ with $x' \in (B_l)^I$. 
As $k_{\emme_{RC}}(x')< i$, 
by inductive hypothesis,  $x' < x$ in $\emme$. 
As $x' \in {B_l}^I$, $k_{\emme}(B_l) \leq k_{\emme}(x')$. 
Since it can be shown that $k_{\emme}(x') < k_{\emme}(B_i) $, $k_{\emme}(B_l) < k_{\emme}(B_i)$, and by condition (b), it holds that $x < y$ in $\emme$. 

%
With these facts,  since $y \in min_<({C'}^I)$  holds in $\emme$, also $y \in min_{<_{RC}}({C'}^{I})$ in $\emme_{RC}$, hence $\tip(C') \sqsubseteq D$ does not hold in $\emme_{RC}$, and $C \sqsubseteq D = \tip(C') \sqsubseteq D \not\in \overline{TBox}$.

The theorem follows by contraposition.
\normalcolor
\end{proof}

 \hide{\begin{proof}
 First recall that by Theorem \ref{Theorem_RC_TBox} $C \sqsubseteq D \in \overline{TBox}$ if $C \sqsubseteq D$ holds in all minimal $\alctr$ canonical models of $K$. In the following we reason by contraposition by showing that 
 if $K \not\models_{\alctre + {{min-canonical}}} C \sqsubseteq D$ then there is a minimal $\alctr$ canonical model of $K$ in which $C \sqsubseteq D$ does not hold, which by what just recalled (Theorem \ref{Theorem_RC_TBox}) entails that $C \sqsubseteq D \not\in \overline{TBox}$

By contraposition suppose that $K \not\models_{\alctre + {{min-canonical}}} C \sqsubseteq D$. 
Then there is a minimal canonical enriched $\alctre$ model  $\emme =  \langle \Delta, <_{A_1}, \dots ,<_{A_n}, <, I  \rangle$  of $K$ and an $y \in C^I$ such that $y \not\in D^I$. 
All consistent sets of concepts consistent with $K$ w.r.t. $\alctre$ are also consistent with $K$ with respect to $\alctr$, and viceversa (by Theorem \ref{equivalence-multiple})
By definition of {\em canonical}, there is also a canonical $\alctr$ model of $K$ with domain ${\Delta}_{RC}=\Delta$ and $I_{RC}$ is defined as $I$ on atomic concepts.  
Let  $\emme_{RC} =  \langle {\Delta}_{RC},  <_{RC}, I_{RC}  \rangle$ be this model. If $C$ does not contain the $\tip$ operator, we are done: also in $\emme_{RC}$ 
$y \in C^{I_{RC}}$ such that $y \not\in D^{I_{RC}}$, hence $C \sqsubseteq D$ does not hold in  $\emme_{RC}$, and $C \sqsubseteq D \not\in \overline{TBox}$.

If $\tip$ occurs in $C$, and $C = \tip(C')$, we still need to show that also in $\emme_{RC}$, as in $\emme$, 
$y \in min_{<_{RC}}({C'}^I)$.  We prove this by showing that for all $x,y \in \Delta$ if $x <_{RC}y$ in $\emme_{RC}$, then also 
$x <y$ in $\emme$ or in other terms that if $k_{\emme_{RC}}(x) < k_{\emme_{RC}}(y)$, then  $k_{\emme}(x) < k_{\emme}(y)$ in $\emme$. 

First of all, notice that by definition of rank of a formula, of the sets $E_0 \dots E_n$ of Proposition \ref{definition_exceptionality} , together with Proposition  13 of \cite{AIJ15}, we have that if 
$k_{\emme_{RC}}(x)= i$, then $x \models E_i$, whereas $x \not\models E_{i-1} - E_i$. By the same definitions and propositions, for all concepts $C$, if rank($C$)$=i$, then $k_{\emme_{RC}}(C)= i$.  We will use these facts in the following. 

The proof is by induction on  $k_{\emme_{RC}}(x)$, the rank  of $x$ in $\emme_{RC}$.
 
(a): let $k_{\emme_{RC}}(x)= 0$ and $k_{\emme_{RC}}(y)>0$. Since $x$ does not violate any inclusion, also in $\emme$ (by minimality of $\emme$) for all preference relations $<_{A_j}$  $k_{{A_j}_\emme}(x) = 0 $, and also  $k_{\emme}(x)= 0$. This cannot hold for $y$, for which $k_{\emme}(y)> 0$ (otherwise $\emme$ would violate 
$K$, against the hypothesis). 

(b): let $k_{\emme_{RC}}(x)= i < k_{\emme_{RC}}(y)$, i.e. $x<_{RC}y$.
As $x <_{RC} y$ in $\emme_{RC}$ and the rank of $x$ in $\emme_{RC}$ is $i$,
there must be a $\tip(B_i) \sqsubseteq A_i \in E_i - E_{i+1}$ such that $x \in (\neg B_i \sqcup A_i)^I$ whereas $y \in (B_i \sqcap \neg A_i)^I$. Before we proceed let us notice that by definition of $E_i$, as well as by what stated just above on the relation between rank of a concept and $k_{\emme_{RC}}$, $k_{\emme_{RC}}(B_i)= k_{\emme_{RC}}(x)$ . We will use this fact below.

We want to prove that condition (b) holds to conclude that $x<y$ (i.e. $k_{\emme}(x) < k_{\emme}(y)$).
We already know that $y$ violates the inclusion $\tip(B_i) \sqsubseteq A_i \in E_i - E_{i+1}$. 
We show that, for any inclusion $\tip(B_l) \sqsubseteq A_l \in K$ that is violated by $x$, it holds that $k_{\emme}(B_l) < k_{\emme}(B_i)$,
so that, by (b), $x<y$.

Let $\tip(B_l) \sqsubseteq A_l \in K$ violated by $x$, i.e. such that  $x \in (B_l \sqcap \neg A_l)^I$.
Since $\emme_{RC}$ satisfies $K$,  there must be $x' <_{RC} x$ with $x' \in (B_l)^I$. 
As $k_{\emme_{RC}}(x')< i$, 
by inductive hypothesis,  $x' < x$ in $\emme$. 
As $x' \in {B_l}^I$, $k_{\emme}(B_l) \leq k_{\emme}(x')$. 
Also, we can show that $k_{\emme}(x') < k_{\emme}(B_i) $ (see below).
Hence, $k_{\emme}(B_l) < k_{\emme}(B_i)$, and by condition (b), it holds that $x < y$ in $\emme$. 

To conclude the proof, let us show that $k_{\emme}(x') < k_{\emme}(B_i) $.
Indeed, as $k_{\emme_{RC}}(B_i)= k_{\emme_{RC}}(x)=i$, we have that
for all $y \in {B_i}^I$, $k_{\emme_{RC}}(y) \geq i$. 
Hence, $k_{\emme_{RC}}(x') < k_{\emme_{RC}}(y)$.
Thus, by inductive hypothesis, $k_{\emme}(x') < k_{\emme}(y)$, for all $y \in {B_i}^I$,
that is, $k_{\emme}(x') < k_{\emme}(B_i)$.

With these facts,  since $y \in min_<({C'}^I)$  holds in $\emme$, also $y \in min_{<_{RC}}({C'}^{I_{RC}})$ in $\emme_{RC}$, hence $\tip(C') \sqsubseteq D$ does not hold in $\emme_{RC}$, and $C \sqsubseteq D = \tip(C') \sqsubseteq D \not\in \overline{TBox}$.

The theorem follows by contraposition.
\normalcolor
\end{proof}
}

\hide{\section{A syntactic characterization: the multipreferential-closure}

In this section we propose a syntactic characterization of our semantics, that we call {\em multipreferential-closure} (MP-closure, for short),
and we compare our approach to lexicographic closure.
The idea is to exploit the rational closure construction, both for reasoning on the whole KB
and for reasoning about the single aspects $A_i$. 
We provide a construction to check the entailment of a subsumption query $\tip(C) \sqsubseteq D$
form a TBox which is polynomial in the size of the TBox and of the query.
We show that computing the multipreferential-closure has the same complexity as reasoning in $\alc$.

Let us recall the definition of rational closure in \cite{AIJ15}, which extends
the notion of rational closure proposed by Lehmann and Magidor \cite{whatdoes} to the logic $\alc$.
Here we only consider rational closure of TBox.

\begin{definition}[Exceptionality of concepts and inclusions]\label{definition_exceptionality}
Let $T_B$ be a TBox and $C$ a concept. $C$ is
said to be {\em exceptional} for $T_B$ if and only if $T_B \models_{\alctr} \tip(\top) \sqsubseteq
\neg C$. A \tip-inclusion $\tip(C) \sqsubseteq D$ is exceptional for $T_B$ if $C$ is exceptional for $T_B$. The set of \tip-inclusions of $T_B$ which are exceptional in $T_B$ will be denoted
as $\mathcal{E}$$(T_B)$.
\end{definition}

\noindent Given a DL  KB=(TBox,ABox),
it is possible to define a sequence of non increasing subsets of
TBox $E_0 \supseteq E_1, E_1 \supseteq E_2, \dots$ by letting $E_0 =\mbox{TBox}$ and, for
$i>0$, $E_i=\mathcal{E}$$(E_{i-1}) \unione \{ C \sqsubseteq D \in \mbox{TBox}$ s.t. $\tip$ does not occurr in $C\}$.
Observe that, being KB finite, there is
an $n\geq 0$ such that, for all $m> n, E_m = E_n$ or $E_m =\emptyset$.

\begin{definition}[Rank of a concept]\label{Def:Rank of a formula_DL} A concept $C$ has {\em rank} $i$ (denoted $\rf(C)=i$) for TBox,
iff $i$ is the least natural number for which $C$ is
not exceptional for $E_{i}$. {If $C$ is exceptional for all
$E_{i}$ then $\rf(C)=\infty$ ($C$ has no rank).}
\end{definition}

As the semantics proposed in the previous section is stronger then rational closure,
given a TBox, we compute the sequence
 $E_0,E_1, \ldots, E_n$ according to the construction above. 
 Also, we consider a similar construction for each of the aspects $A_i$  and we let: 
 $E^{A_i}_0,E^{A_i}_1, \ldots, E^{A_i}_{n_i}$  sequence obtained by applying the construction above to 
 $TBox^{A_i}=\{ \tip(C) \sqsubseteq (\neg) A_i \in TBox\} \cup \{ C \sqsubseteq D \in \mbox{TBox}$: $\tip$ does not occurr in $C\}$.
 
 Given a query $\tip(B) \sqsubseteq D$, let $\rf(B)=k$.
 The idea is that, among the $B$-domain elements of rank $k$  in any canonical enlarged model of TBox, 
 we want to select those elements $x$ preferred according to some aspect $A_i$, 
 and such that no other  $B$-element of rank $i$ is preferred to $x$  with respect to another aspect $A_j$
 (according to condition (b), in Definition\ref{def-enrichedmodelR}).
Syntactically, we look for the largest set $E^{A_i}_{m^i}$ of defaults concerning aspect $A_i$
which is compatible with the inclusions in $E_k$ and with the fact that there is at least some $B$-element.

Hence, we let ${m^i} \in \{1,\ldots, n_i\}$ be the smallest integer such that: $E^{A_i}_{m^i} \cup E_k \not \models \tip(\top) \sqsubseteq \neg B$.
The $B$-elements of rank $k$ satisfying the defaults in  $E^{A_i}_{m^i}$ are those preferred among the $B$ elements with respect to aspect $A_i$.
Of course, $E^{A_i}_{m^i}$ might contain default inclusions non already present in $E_k$ or it might not.
For ${\cal L}_A=\{A_1,\ldots,A_r\}$, we let 
\begin{center}
$\Delta_0 = E^{A_1}_{m^1}  \cup \ldots \cup E^{A_r}_{m^r}$
\end{center}
The minimal domain elements satisfying the defaults in $\Delta_0 \cup E_k$ are the domain elements of rank $k$
preferred with respect to all of the aspects. This set is non-empty due to condition $(vi)$ in Definition\ref{def-enrichedmodelR}. 
However, $\Delta_0$ might not contain $B$ elements at all,
i.e.,
$\Delta_0 \cup E_k \models \tip(\top) \sqsubseteq \neg B$. 
In this case, we need to consider separately those B elements that are preferred with respect to maximal subsets of aspects
(as some $B$-elements may be preferred with respect to some set of aspects, while another $B$-element can be preferred with respect to another set of aspects, and the two elements may be non-comparable in the global relation $<$).
When $\Delta_0 \cup E_k \not\models \tip(\top) \sqsubseteq \neg B$, we can consider the $B$-elements satisfying all the defaults in $\Delta_0$ (which are the $B$-elements preferred with respect to all the aspects).  

\begin{proposition}
Let  $\tip(B) \sqsubseteq D$ be a query and let $k=\rf(B)$ is the rank of concept $B$ in the rational closure.
$\tip(B) \sqsubseteq D$ {\em is derivable from the MP-closure of TBox} if  
for all the maximal subsets  $\Delta$ of $\{ E^{A_1}_{m^1}, \ldots, E^{A_r}_{m^r} \}$ 
such that $\bigcup \Delta \cup E_k \not \models \tip(\top) \sqsubseteq \neg B$, we have that:
$$\bigcup \Delta \cup E_k \models_{\alctr} \tip(\top) \sqsubseteq (\neg B \sqcup D)$$
\end{proposition}



\begin{proposition}
Given a KB and a query $\tip(B) \sqsubseteq D$,
$K \models_{\alctre,{min_c}} \tip(B) \sqsubseteq D$
if and only if $\tip(C) \sqsubseteq D$ {\em is derivable from the MP-closure of TBox}
\end{proposition}

\begin{center}
$ \tip(B) \sqsubseteq D \in \overline{TBox}$ iff $E_k \models_{\alctr} \tip(\top) \sqsubseteq (\neg B \sqcup D)$.
\end{center}

(Only if) Supponiamo che $\rf(B )<\rf(B \sqcap \neg D)$.
Vogliamo mostrare che $E_k \models_{\alctr} \tip(\top) \sqsubseteq (\neg B \sqcup D)$.
Sia $\emme \models_{\alctr} E_k$.
Supponiamo che esista un $B$-elemento $x$ di rango $0$ in $\emme$ (se non esiste, abbiamo la tesi).
Se $x$ fosse un $B \sqcap \neg D$, avrei 
$E_k \not \models \tip(\top) \sqsubseteq \neg (B \sqcap \neg D)$, contro l'Hp che $\rf(B \sqcap \neg D)>k$.
Dunque $x$ deve essere un $D$ e $\emme \models_{\alctr} \tip(\top) \sqsubseteq  (\neg B \sqcup D)$.

(If) Supponiamo che $E_k \models_{\alctr} \tip(\top) \sqsubseteq (\neg B \sqcup D)$.
Vogliamo mostrare che: $\rf(B )<\rf(B \sqcap \neg D)$.
Dato che $\rf(B)=k$ sappiamo costruire un modello canonico minimale di TBox i cui $B$ ha rango $k$,
partendo da un modello di $E_k$ che contiene un $B$ elemento $x$ di rango 0.  ma per Hp ogni $x$ deve essere un 
$\neg B \sqcup D$ e quindi un $D$. Dalla costruzione otterro' un modello canonico minimale in cui 
i $B$ minimali sono $D$ e hanno rango $k$. Dunque in tale modello $k_\emme(B \sqcap \neg D)>k$.
Allora, per le proprieta' di rational closure, deve essere $\rf(B \sqcap \neg D)>k$. ]]

\normalcolor
While entailment in $\alctr$ can be computed in  \textsc{ExpTime} \cite{AIJ15},
verifying if a query $\tip(B) \sqsubseteq D$ holds in all the minimal canonical enriched  models of the TBox
in the worst case requires to consider an exponential number ($2^r$) of subsets $\Delta$ 
of $\{ E^{A_1}_{m^1}, \ldots, E^{A_r}_{m^r} \}$,
to find the maximal $\Delta$ such that $\bigcup \Delta \cup E_k \models_{\alctr} \tip(\top) \sqsubseteq (\neg B \sqcup D)$.
Hence computing MC-closure has the same complexity as computing the lexicographic colsure
(although in the last case, the bases are compared using the lexicographic ordering).


Let us consider again Example \ref{example-global-relation}.
\begin{example}
Let $TBox = \{Penguin \sqsubseteq Bird, \tip(Bird) \sqsubseteq HasNiceFeather$, $\tip(Bird) \sqsubseteq Fly$,  $\tip(Penguin) \sqsubseteq \neg Fly \}$.
We have:
\begin{quote}
 $TBox^{Fly} = \{Penguin \sqsubseteq Bird$, $\tip(Bird) \sqsubseteq Fly$,  $\tip(Penguin) \sqsubseteq \neg Fly \}$ \\
 $TBox^{HasNiceFeather} = \{Penguin \sqsubseteq Bird, \tip(Bird) \sqsubseteq HasNiceFeather$\} \\
 $TBox^{Bird} = TBox^{Penguin} =  \{Penguin \sqsubseteq Bird\} $
 \end{quote}
We want to check if the query: $\tip(Penguin) \sqsubseteq HasNiceFeather$ holds in all the minimal enriched canonical models of TBox.
From the rational closure of TBox, we have: $\rf(Bird)=0$,  $\rf(Penguin)=1$ and $E_1=  \{Penguin \sqsubseteq Bird$, $\tip(Penguin) \sqsubseteq \neg Fly \}$. Furthermore, we get
$\Delta_0 = E^{Fly}_{1}  \cup E^{HasNiceFeather}_{0}  \cup E^{Bird}_{0} \cup E^{Penguin}_{0}$, where:
\begin{quote}
 $ E^{Fly}_{1} = \{Penguin \sqsubseteq Bird$,  $\tip(Penguin) \sqsubseteq \neg Fly \}$ \\
 $E^{HasNiceFeather}_{0} = \{Penguin \sqsubseteq Bird, \tip(Bird) \sqsubseteq HasNiceFeather$\} \\
 $E^{Bird}_{0} =  E^{Penguin}_{0}=  \{Penguin \sqsubseteq Bird\} $.
 \end{quote}
As $\Delta_0 \cup E_1 \not \models \tip(\top) \sqsubseteq \neg Penguin$,  $\Delta= \Delta_0$ is the unique maximal subset
of $\{E^{Fly}_{1}  , E^{HasNiceFeather}_{0}  , E^{Bird}_{0} , E^{Penguin}_{0}\}$
It holds that,  $\Delta_0 \cup E_1 \models \tip(\top) \sqsubseteq (\neg$ $ Penguin \sqcup HasNiceFeather)$,
so that $\tip(Penguin) \sqsubseteq HasNiceFeather $ is derivable from the MP-closure, \normalcolor
which is in agreement with the fact that in all the minimal enriched canonical models of TBox the typical penguins have nice feather.
\end{example} 



\vspace{-0.2cm}
} 

\section{Conclusions and Related Works}

A lot of work  has  been done in order to extend the basic formalism of Description Logics (DLs)  with nonmonotonic reasoning features \cite{Straccia93,baader95b,donini2002,eiter2004,lpar2007,AIJ13,kesattler,sudafricaniKR,bonattilutz,casinistraccia2010,rosatiacm,hitzlerdl,KnorrECAI12,CasiniDL2013}. 
The purpose of these extensions is to allow reasoning about \emph{prototypical properties} of individuals or classes of individuals. 

The interest of rational closure for DLs is that it 
provides a significant and reasonable skeptical nonmonotonic inference mechanism, 
while keeping the same complexity as the underlying logic.
The first notion of rational closure for DLs was defined by Casini and Straccia \cite{casinistraccia2010}. 
Their rational closure construction for $\alc$  directly uses entailment in $\alc$ over a materialization of the KB.
A variant of this notion of rational closure has been studied in \cite{CasiniDL2013}, and a semantic characterization for it has been proposed.
In \cite{dl2013,AIJ15} a notion of rational closure for the logic $\alc$ has been proposed, building on the notion of rational closure proposed by Lehmann and Magidor \cite{whatdoes}, together with a minimal model semantics characterization.

It is well known that rational closure has some weaknesses that accompany its well-known qualities, both in the context of propositional logic and in the context of Description Logics. 
Among the weaknesses is the fact that one cannot separately reason property by property,
so that, if a subclass of $C$ is exceptional for a given aspect, it is exceptional ``tout court'' and does not inherit any of the
typical properties of $C$.
Among the strengths of rational closure there is its computational lightness,  which is crucial in Description Logics. To overcome the limitations of rational closure,  
in \cite{Casinistraccia2011,CasiniJAIR2013} an approach is introduced based on the combination of rational closure and \emph{Defeasible Inheritance Networks},  while in  \cite{Casinistraccia2012} a lexicographic closure is proposed, and in \cite{Casini2014} relevant closure, a syntactic stronger version  of rational closure.
To address the mentioned weakness of rational closure, in this paper we have proposed a finer grained semantics of the semantics for rational closure proposed in \cite{AIJ15}, where  models are equipped with several preference relations. In such a semantics it is possible to relativize the notion of typicality, whence to reason about typical properties independently from each other. We are currently working at the formulation of a syntactic characterization of the semantics which will be a strengthening of rational closure.
As the semantics we have proposed provides a strengthening of rational closure,
a natural question arises whether this semantics is equivalent to the lexicographic closure proposed in \cite{Lehmann95}.
In particular,  lexicographic closure construction for the description logic $\alc$ has been defined in \cite{Casinistraccia2012}.
Concerning our Example \ref{esempio-strengthening of rational closure} above, our minimal model semantics gives the same results as lexicographic closure,
since $\tip(Penguin) \sqsubseteq HasNiceFeather$ can be derived from the lexicographic closure of the TBox
and $\tip(Penguin) \sqsubseteq HasNiceFeather$ holds in all the minimal canonical enriched  models of TBox.
However, a general relation needs to be established.

An approach related to our approach is given in \cite{fernandez-gil}, where it is proposed an extension of $\alct$ with several typicality operators, each corresponding to a preference relation.
This approach is related to ours although  different: the language in  \cite{fernandez-gil}  allows for several typicality operators whereas  we only have a single typicality operator.  The focus of  \cite{fernandez-gil}  is indeed different from ours, as it does not deal with rational closure, whereas this is the main contribution  of our paper.


\normalcolor

{\bf Acknowledgement:} This research is partially supported by INDAM-GNCS Project 2016 "Ragionamento Defeasible nelle Logiche Descrittive".
\bibliographystyle{aaai}
\bibliography{biblioCilc15bis}
\end{document}